\documentclass{article}
\pdfoutput=1

\usepackage[preprint,nonatbib]{neurips_2019}

\usepackage[utf8]{inputenc} 
\usepackage[T1]{fontenc}    
\usepackage{hyperref}       
\usepackage{url}            
\usepackage{booktabs}       
\usepackage{amsfonts}       
\usepackage{nicefrac}       
\usepackage{microtype}      

\usepackage{amsmath,amsfonts,amssymb,amsthm,mathrsfs}
\usepackage{graphicx,graphics}
\usepackage{epsfig}
\usepackage{latexsym} 
\usepackage[english]{babel}
\usepackage[utf8]{inputenc}
\usepackage{mathtools}
\usepackage{enumerate}
\usepackage{xcolor}
\usepackage{lmodern}
\usepackage{bm}
\usepackage{notations}
\newtheorem{theorem}{Theorem}

\newtheorem{lemma}{Lemma}
\newtheorem{proposition}{Proposition}

\title{Asymptotic Bayes risk for Gaussian mixture in a semi-supervised setting}

\author{%
  Marc Lelarge \\
  INRIA-ENS\\
  Paris, France \\
  \texttt{marc.lelarge@ens.fr} \\
  \And
  L\'eo Miolane \\
  INRIA-ENS \\
  Paris, France \\
  \texttt{leo.miolane@inria.fr} \\
}

\begin{document}

\maketitle

\begin{abstract}
  Semi-supervised learning (SSL) uses unlabeled data for training and has been shown to greatly improve performance when compared to a supervised approach on the labeled data available. This claim depends both on the amount of labeled data available and on the algorithm used. 
  In this paper, we compute analytically the gap between the best fully-supervised approach using only labeled data and the best semi-supervised approach using both labeled and unlabeled data. We quantify the best possible increase in performance obtained thanks to the unlabeled data, i.e. we compute the accuracy increase due to the information contained in the unlabeled data. Our work deals with a simple high-dimensional Gaussian mixture model for the data in a Bayesian setting. Our rigorous analysis builds on recent theoretical breakthroughs in high-dimensional inference and a large body of mathematical tools from statistical physics initially developed for spin glasses.
\end{abstract}

\section{Introduction}

Semi-supervised learning (SSL) has proven to be a powerful paradigm for leveraging unlabeled data to mitigate the reliance on large labeled datasets. The goal of SSL is to leverage large amounts of unlabeled data to improve the performance of supervised learning over small datasets. For unlabeled examples to be informative, assumptions have to be made. The cluster assumption states that, if two samples belong to the same cluster in the input distribution, then they are likely to belong to the same class. The cluster assumption is the same as the low-density separation assumption: the decision boundary should lie in the low-density region.

In this paper, we explore analytically the simplest possible parametric model for the cluster assumption: the two clusters are modeled by a mixture of two high-dimensional Gaussians with diagonal covariance so that the optimal decision boundary is a hyperplane. Our model can be seen as a classification problem in a semi-supervised setting. Our aim here is to define a model simple enough to be mathematically tractable while being practically relevant and capturing the main properties of a high-dimensional statistical inference problem.

Our model has three parameters: the high-dimensionality of the data is captured by $\alpha$ the ratio of the number of samples divided by the ambient dimension; the fraction of labeled data point $\eta$ and the amount of overlap between the clusters $\sigma^2$.
As a function of these three parameters, we compute the best possible accuracy (the Bayes risk) when only labeled data are used or when unlabeled data are also used. As a result, we obtain the added value due to the unlabeled data for the best possible algorithm. In particular, we observe a very clear diminishing return of the labeled data, i.e. the first labeled data points bring much more information than the last ones. Hence the regime with very few labeled data points is a priori a regime favorable to SSL. But in this case, we face in practice the problem of small validation sets \cite{oliver2018realistic} which makes hyperparameter tuning impossible.

We find that the range of parameters for which SSL clearly outperforms either unsupervised learning or supervised learning on the labeled data is rather narrow.
In a case with large overlap between the clusters ($\sigma^2\to \infty$), unsupervised learning fails and supervised learning on the labeled data is almost optimal. In a case with small overlap between the clusters ($\sigma^2\to 0$), unsupervised learning achieves performances very close to supervised learning with all labels available while using only the labeled dataset fails.

From a practical perspective, we can try to draw parallels between our results and the state of the art in SSL but we need to keep in mind that our results only give best achievable performances on our toy model. In particular, even in a setting where our results predict that unsupervised learning achieves roughly the same performances as supervised learning with all labels, it might be very useful in practice to use a few labels in addition to all unlabeled data. Such an approach is 
presented in \cite{berthelot2019mixmatch} where extremely good performances are achieved for image classification with only a few labeled data per class and a new SSL algorithm: MixMatch. For example on CIFAR-10, with only 250 labeled images, MixMatch achieves an error rate of 11.08\% and with 4000 labeled images, an error rate of 6.24\% (to be compared with the 4.17\% error rate for the fully supervised training on all 50000 samples). These results are aligned with our finding about diminishing returns of labeled data points.

We make the following contributions:\\
{\bf Bayes risk:} to the best of our knowledge, our work is the first analytic computation of the Bayes risk in a high-dimensional Gaussian model in a semi-supervised setting.

{\bf Rigorous analysis:} our analysis builds on a series of recent works \cite{deshpande2016asymptotic,dia2016mutual,lelarge2019fundamental,miolane2017fundamental,barbier2019optimal} with tools from information theory and mathematical physics originally developed for the analysis of spin glasses \cite{panchenko2013sherrington,talagrand2010mean}.

The rest of the paper is organized as follows. Our model and the main result are presented in Section \ref{sec:model}. Related work is presented in Section \ref{sec:related}. In Section \ref{sec:proof}, we give an heuristic derivation of the main result and in Section \ref{sec:sketch}, we give a proof sketch  while the more technical details are presented in the supplementary material Section \ref{sec:supp}. We conclude in Section \ref{sec:con}

\section{Model and main results}\label{sec:model}
We now define our classification problem with two classes. The points $\bbf{Y}_1, \dots, \bbf{Y}_N$ of the dataset are in $\R^D$ and given by the following process:
$$
\bbf{Y}_j
=
V_j \bbf{U} + \sigma \bbf{Z}_j,\quad 1\leq j\leq N,
$$
where $\bbf{U} \sim \Unif(\mathbb{S}^{D-1})$, 
$\bbf{V} = (V_1, \dots, V_N) \iid \text{Unif}(-1,1)$
and $\bbf{Z}_1, \dots, \bbf{Z}_N \iid \cN(0,\text{Id}_D)$ are all independent.

In words, the dataset is composed of $N$ points in $\R^D$ divided into two classes with roughly equal sizes. The points with label $V_j = +1$ are centered around $+\bbf{U}\in \R^D$ and the points with label $V_j = -1$ are centered around $-\bbf{U}\in \R^D$. The parameter $\sigma$ controls the level of Gaussian noise around these centers.

In a semi-supervised setting, the statistician has access to some labels. We consider a case where each label is revealed with probability $\eta\in [0,1]$ independently of everything else. To fix notation, the side information is given by the following process:
$$
S_j = 
\begin{cases}
	V_j & \text{with probability} \ \eta \\
	0 & \text{with probability} \ 1-\eta.
\end{cases}
$$
If $S_j=0$, then the label of the $j$-th data point is unknown whereas if $S_j=\pm1$, it corresponds to the label of the $j$-th data point.

Finally, we consider the high-dimensional setting and all our results will be in a regime where $N,D\to \infty$ while the ratio $N/D$ tends towards a constant $\alpha >0$. Note that we are in a high noise regime since the squared norm of the signal is one whereas the squared norm of the noise is $\sigma^2 D \approx \sigma^2 N/\alpha$ where $N$ is the number of observations. 

To summarize, the three parameters of our model are: $\sigma^2>0$ the variance of the noise in the dataset, $\eta\in [0,1]$ the fraction of revealed labels and $\alpha>0$ the ratio between the number of data points (both labeled and unlabeled) and the dimension of the ambient space. We also assume that the statistician knows the priors, i.e. the distribution of $\bbf{U},\bbf{V}$ and $\bbf{Z}$.

The task of the statistician is to use the dataset $(\bbf{Y},\bbf{S})$ in order to make a prediction about the label of a new (unseen) data point. More formally, we define:
$$
\bbf{Y}_{\rm new} = V_{\rm new} \bbf{U} + \sigma\bbf{Z}_{\rm new},
$$
where $V_{\rm new} \sim \text{Unif}(-1,+1)$, $\bbf{Z}_{\rm new} \sim \cN(0,\text{Id}_D)$. We are interested in the minimal achievable error in our model, i.e. the Bayes risk:
$$
R^*_D(\eta)= \inf_{\hat{v}} \P\big(\what{v}(\bbf{Y}, \bbf{S}, \bbf{Y}_{\rm  new}) \neq V_{\rm new} \big)
$$
where the infimum is taken over all estimators (measurable functions of $\bbf{Y},\bbf{S},\bbf{Y}_{\rm new}$).

Our main mathematical achievement is an analytic formula for the Bayes risk $R^*_D$ in the large $D$ limit, see Theorem \ref{th:main} below. In order to state it, we need to introduce some additional notation. We start with some easy facts about our model.

\paragraph{Oracle risk}
Assume that the statistician knows the center of the clusters, i.e.\ has access to the ``oracle'' vector $\bbf{U}$. Then the best classification error would be achieved thanks to the simple thresholding rule $\sign(\langle \bbf{U}, \bbf{Y}_{\rm new}\rangle)$, where $\langle ., . \rangle$ denotes the Euclidean dot product. In this case, the risk is given by:
\begin{eqnarray}\label{eq:oracle}
	R_{\rm oracle} = \P\big( \sigma \langle \bbf{U},\bbf{Z}_{\rm new}\rangle>1\big) = \P\big(\sigma Z>1) = 1-\Phi\left( \frac{1}{\sigma}\right),
\end{eqnarray}
where $\Phi$ is the standard Gaussian cumulative distribution function.
We have of course $R_{\rm oracle} \leq R^*_D(\eta)$.

\paragraph{Fully supervised case}
Another instructive and simple case is the supervised case where $\eta =1$. Since all the $V_j$'s are known, we can assume wlog that they are all equal to one (multiply each $\bbf{Y}_j$ by $V_j$). More importantly, if we slightly modify the distribution of $\bbf{U}$ by taking $\bbf{U}=(U_1,\dots U_D)\iid \cN(0,1/D)$, this will not change the results for our model and makes the analysis easier by decorrelating each component. Indeed, denote by $Y_j$ (resp. $Z_j$) the first component of $\bbf{Y}_j$ (resp. $\bbf{Z}_j$) and by $U_1$ the first component of $\bbf{U}$. Then we have $N$ scalar noisy observations of the first component of $\bbf{U}$: $Y_j = U_1+\sigma Z_j$ for $1\leq j\leq N$, so that we can construct an estimate for $U_1$ by taking the average of the observations. We get:
$$
\overline{Y}_1 = \frac{1}{N}\sum_{j=1}^N Y_j = U_1+\frac{\sigma}{\sqrt{N}}\cN(0,1).
$$
Doing this for each component of $\bbf{U}$, we get an estimate of the vector $\bbf{U}$ and we now use it to get an estimate of $V_{\rm new}$. First define $\overline{\bbf{Y}} = (\overline{Y}_1,\dots, \overline{Y}_D)$ and consider
$$
\langle \bbf{Y}_{\rm new}, \overline{\bbf{Y}}\rangle = V_{\rm new}\langle \bbf{U}, \overline{\bbf{Y}}\rangle + \sigma\langle \bbf{Z}_{\rm new},\overline{\bbf{Y}}\rangle,
$$
and note that as $D\to \infty$, we have $\langle \bbf{U}, \overline{\bbf{Y}}\rangle \approx D\E[U_1\overline{Y}_1]=D\E\left[U_1^2\right]=1$ and $\langle \bbf{Z}_{\rm new},\overline{\bbf{Y}}\rangle \approx \sqrt{\E\left[ \|\overline{\bbf{Y}}\|^2\right]} Z\approx  \sqrt{\alpha+\sigma^2}/\sqrt{\alpha} \cN(0,1)$, so that we get:
$$
\langle \bbf{Y}_{\rm new}, \overline{\bbf{Y}}\rangle \approx V_{\rm new} + \frac{\sigma\sqrt{\alpha+\sigma^2}}{\sqrt{\alpha}} \cN(0,1).
$$
Our main result will actually show that estimating $V_{\rm new}$ with the sign of $\langle \bbf{Y}_{\rm new}, \overline{\bbf{Y}}\rangle$ is optimal so that we get:
\begin{equation}
\label{eq:risksupervised}
\lim_{N,D \to \infty}R^*_D(1) = \P\left(\frac{\sigma\sqrt{\alpha+\sigma^2}}{\sqrt{\alpha}}Z>1\right) = 1-\Phi\left( \frac{\sqrt{\alpha}}{\sigma\sqrt{\alpha+\sigma^2}}\right).
\end{equation}
A similar result was obtained in \cite{dobwag18} in a case where the covariance structure of the noise needs also to be estimated by the statistician resulting in a multiplicative term inside the $\Phi(.)$ function (see Theorem 3.1 and Corollary 3.3 in \cite{dobwag18}).

\paragraph{Unsupervised case}
In this paper, we concentrate on the case where $\eta>0$. When $\eta =0$, there is no side information and we are in an unsupervised setting studied in \cite{miolane2017fundamental}. Due to the symmetry of our model, we have $R^*_D=1/2$ because there is no way to guess the right classes $\pm 1$. In order to have a well-posed problem, the risk should be redefined as follows:
$$
R^*_D(0) =  \E_{\bbf Y}\left[\min_{s=\pm 1}\inf_{\hat{v}}\P\left(s\hat{v}(\bbf{Y}, \bbf{Y}_{\rm  new}) \neq V_{\rm new}|\bbf{Y}\right) \right]
$$
Although, this measure of performance is not the one studied in \cite{miolane2017fundamental}, we can adapt the argument to show that:
\begin{equation}
\label{eq:riskunsupervised}
\lim_{\eta\to 0}\lim_{N,D\to \infty}R^*_D(\eta) = \lim_{N,D\to \infty}R^*_D(0).
\end{equation}

\paragraph{Main result}
We now state our main result:
\\

\begin{theorem}\label{th:main}
	Let us define, for $\alpha, \sigma >0$, $\eta \in (0,1]$,
	\begin{equation}\label{eq:def_f}
	f_{\alpha,\sigma,\eta}(q)
	\defeq
	\alpha (1-\eta) \mathsf{i}_v(q / \sigma^2) + \frac{\alpha}{2\sigma^2}(1-q) - \frac{1}{2}\big(q  + \log(1-q)\big).
	\end{equation}
	Here $\mathsf{i}_v(\gamma) = \gamma - \E \log \cosh(\sqrt{\gamma} Z_0 + \gamma)$ where $Z_0 \sim \cN(0,1)$. 
	The function $f_{\alpha,\sigma,\eta}$ admits a unique minimizer $q^*(\alpha,\sigma,\eta)$ on $[0,1)$ and
	\begin{equation}
	R^*_D(\eta)
	\xrightarrow[ N,D \to \infty ]{} 1-\Phi(\sqrt{q^*(\alpha,\sigma,\eta)}/\sigma).
	\label{eq:Risk}
        \end{equation}
\end{theorem}

As a result from our proof, we will prove that a very simple algorithm is optimal (asymptotically in $N,D$). Namely, we define $\widebar{\bbf{u}} = \E[\bbf{U}|\bbf{Y},\bbf{S}]$, the posterior mean of $\bbf{U}$ given the observations. Then taking $ \what{v} = \sign(\langle \overline{\bbf{u}}, \bbf{Y}_{\rm new}\rangle)$ is an optimal estimator of $V_{\rm new}$ in the sense that
\begin{eqnarray*}
\P (V_{\rm new} \neq \what{v}) = R^*_D(\eta) + o_N(1).
\end{eqnarray*}
Of course, from a practical point of view, computing an estimate for $\widebar{\bbf{u}}= \E[\bbf{U}|\bbf{Y},\bbf{S}]$ is not an easy task (except in the supervised setting) and approximations need to be made.

\begin{figure}[thpb]
	\hspace{-2.1cm}
	\includegraphics[width=1.3 \linewidth]{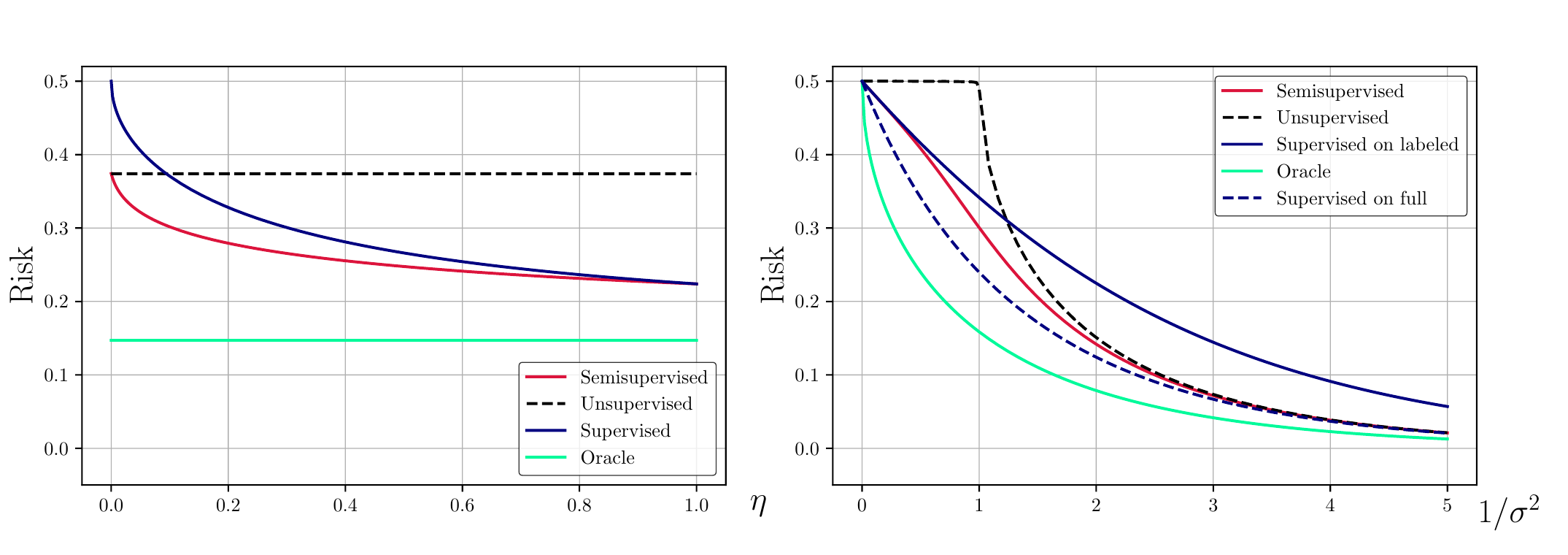}
	\caption{Left: Bayes risks as a function of the fraction of labeled data $\eta$, with noise $\sigma^2=0.9$ and ambient dimension equals to the number of samples $\alpha=1$. Right: Bayes risks as a function of the inverse of the noise $1/\sigma^2$, with fraction of labeled data $\eta=0.2$, and ambient dimension equals to the number of samples $\alpha=1$.}
	\label{fig:comp}
\end{figure}

Figure \ref{fig:comp} gives examples of our main results with comparison of the various settings. The semi-supervised curve corresponds to the formula \eqref{eq:Risk} where a fraction $\eta$ of the data points have labels and are used with all unlabeled data points. The supervised on full curve corresponds to \eqref{eq:risksupervised} where all the labels are used. The supervised on labeled curve corresponds to \eqref{eq:risksupervised} with the parameter $\alpha$ replaced by $\alpha\eta$ and is the best possible performance when only a fraction $\eta$ of the data points having labels are used. The unsupervised curve corresponds to \eqref{eq:riskunsupervised} where all the data points are used but without any label. Finally, the oracle curve corresponds to \eqref{eq:oracle} where the centers of the clusters are known (corresponding to the case $\alpha\to \infty$). In the left of Figure \ref{fig:comp}, we clearly see that the first labeled data points (i.e.\ when $\eta$ is small) decreases greatly the risk of semi-supervised learning. This corresponds to the diminishing return of the labeled data.
In the right plot of Figure \ref{fig:comp}, we see that in the high-noise regime, unsupervised learning fails and that its risk decreases as soon as $\sigma^2<1$. This phenomena is known as the BBP phase transition \cite{baik2005phase,baik2006eigenvalues,paul2007asymptotics}. We see that below this transition, the unlabeled data are of little help as the performance of SSL almost match the performance of supervised learning on labeled data only. Moreover after the transition, unsupervised learning reaches quite quickly the performance of SSL. In other words, the regime most favorable to SSL in term of noise corresponds precisely to the regime around the BBP phase transition where unsupervised learning is still not very good while supervised learning on labeled data saturates.

\section{Related work}\label{sec:related}
The unsupervised version of our problem is the standard Gaussian mixture model used in statistics \cite{friedman2001elements}. In the regime considered here (dimension and number of samples tending to infinity), there are a number of recent works dealing with the clustering problem of Gaussian mixtures. However, a large part of them considers scenarios where $\alpha \to \infty$ or $\sigma \to 0$.
In the regime where $\alpha = O(1)$ i.e.\ where the number of observations is proportional to the dimension, spectral clustering has been extensively studied. In this regime it is known that the leading eigenvector of the sample covariance matrix encounters a phase transition \cite{baik2005phase,baik2006eigenvalues,paul2007asymptotics}: there exists a critical value of the noise intensity below which the leading eigenvector starts to be correlated with the centers of the clusters.
Using  exact  but  non-rigorous  methods  from  statistical physics, \cite{barkai1994statistical,lesieur2016phase} determine the critical values for $\alpha$ and $\sigma$ at which it becomes information-theoretically possible  to  reconstruct  the  membership  into  clusters  better than chance. 
Rigorous results on this model are given in \cite{banks2018information} where bounds on the critical values are obtained. The precise thresholds were then determined in \cite{miolane2017fundamental}. Our analysis builds on the techniques derived in this last reference with two main modifications: additional work is required to compute the classification accuracy (as opposed to the mean squared error) and to incorporate the side information. 

To the best of our knowledge, there are much fewer theoretical works dealing with a semi-supervised learning in a high-dimensional setting. \cite{cast95} shows the error converges exponentially fast in the number of labeled examples if the mixture model is identifiable (see \cite{singh09} for an extension of these results). \cite{castelli1996relative} studies a mixture model where the estimation problem is essentially reduced to the one of estimating the mixing parameter and shows that the information content of unlabeled examples decreases as classes overlap. \cite{coz03} shows that unlabeled data can lead to an increase in classification error in a case where the model is incorrect. A similar conclusion is obtained in \cite{kri18} for linear classifiers defined by convex margin-based sur-rogate losses. In contrast, our work computes the asymptotic Bayes risk for which unlabeled data can only improve the best achievable performance.
More closely related to our work, \cite{NIPS2018_8077}
provides the first information theoretic tight analysis for inference of latent community  structure  given  a  dense graph along with high dimensional node covariates, correlated with the same latent communities. \cite{mai2018random} studies a class of graph-oriented semi-supervised learning algorithms in the limit of large and numerous data similar to our setting.

In contrast, there are a number of practical works and proposed algorithms for semi-supervised learning based on transductive models \cite{joachims2003transductive}, graph-based method \cite{zhu2003semi} or generative modeling \cite{belkin2002laplacian}, see the surveys \cite{zhu2005semi} and \cite{chapelle2009semi}. SSL methods based on training a neural network by adding an additional loss term to ensure consistency regularization are presented in \cite{grandvalet2005semi}, \cite{kingma2014semi}, \cite{verma2019interpolation}.
We refer in particular to the recent work \cite{oliver2018realistic} for an overview of these SSL methods (currently the state-of-the-art for SSL on image classification datasets). The algorithm MixMAtch introduced in \cite{berthelot2019mixmatch} obtains impressive results on all standard image benchmarks. Given these recent improvements, natural questions arise: what is the best possible achievable performance? to what extend can we generalize those improvement to other domains? We believe that our work is a first step in a theoretical understanding of these questions.

\section{Heuristic derivation of the main result} \label{sec:proof}
We present now an heuristic derivation of our results, based on the ``cavity method'' \cite{mezard1987spin} from statistical physics.
Let $\widebar{\bbf{u}} = \E[\bbf{U}|\bbf{Y},\bbf{S}]$ and $\widebar{\bbf{v}} = \E[\bbf{V}|\bbf{Y},\bbf{S}]$ be the optimal estimators (in term of mean squared error) for estimating $\bbf{U}$ and $\bbf{V}$.
A natural hypothesis is to assume that the correlation $\langle \widebar{\bbf{u}}, \bbf{U} \rangle$ converges as $N,D \to \infty$ to some deterministic limit $q_u^* \in [0,1]$ and that $\frac{1}{N} \langle \widebar{\bbf{v}}, \bbf{V} \rangle \to q_v^* \in \R$.

The conditional expectation $\widebar{\bbf{u}} = \E[\bbf{U}|\bbf{Y},\bbf{S}]$ is the orthogonal projection (in $L^2$ sense) of the random vector $\bbf{U}$ onto the subspace of $\bbf{Y},\bbf{S}$-measurable random variables. The squared $L^2$ norm of the projection $\widebar{\bbf{u}}$ is equal to the scalar product of the vector $\bbf{U}$ with its projection $\widebar{\bbf{u}}$: $\E \|\widebar{\bbf{u}} \|^2 = \E \langle \widebar{\bbf{u}}, \bbf{U} \rangle$. Assuming that $\| \widebar{\bbf{u}} \|^2$ also admits a deterministic limit, this limits is then equal to $q_u^*$. We get for large $N$ and $D$, $\|\widebar{\bbf{u}}\|^2 \simeq \langle \widebar{\bbf{u}}, \bbf{U} \rangle \simeq q_u^*$. Analogously we have $\frac{1}{N} \| \widebar{\bbf{v}} \|^2 \simeq \frac{1}{N} \langle \widebar{\bbf{v}}, \bbf{V} \rangle \simeq q_v^*$.

We will show below that $q_u^*$ and $q_v^*$ obey some fixed point equations that allow to determine them.

As seen above, if we aim at estimating a label $V_i$ that we did not observe (i.e.\ $S_i = 0$) given $\bbf{Y}, \bbf{S}$ and the ``oracle'' $\bbf{U}$, we compute the sufficient statistic $\widetilde{Y}_i = \langle \bbf{Y}_i, \bbf{U} \rangle =  V_i + \sigma \cN(0,1)$. 
The estimator that minimizes the probability of error $\P(\what{v} \neq V_i)$ is simply $\what{v}_i = \sign(\widetilde{Y}_i)$. The one that minimizes the mean squared error (MSE) is $\what{v}_i = \E[V_i|\widetilde{Y}_i]$ which achieves a MSE of
$$
\E [(V_i - \what{v}_i)^2] = \mmse_v(1/\sigma^2)
$$
where we define for $(V,Z) \sim \Unif(-1,+1) \otimes \cN(0,1)$ and $\gamma >0$ (see Section \ref{sec:gauss} for more details):
$$
\mmse_v(\gamma) \defeq \E\big[(V - \E[V|\sqrt{\gamma} V + Z ])^2\big].
$$
In the case where we do not have access to the oracle $\bbf{U}$, one can still use $\widebar{\bbf{u}}$ as a proxy. We repeat the same procedure assuming that $\langle \widebar{\bbf{u}}, \bbf{Y}_i \rangle$ is a sufficient statistic for estimating $V_i$. 
Although this is not strictly true, we shall see that this leads to the correct fixed point equations for $q_u^*,q_v^*$.
Compute
$$
\langle \widebar{\bbf{u}}, \bbf{Y}_i \rangle 
\ = \ \langle \widebar{\bbf{u}}, \bbf{U} \rangle V_i + \sigma \langle \widebar{\bbf{u}}, \bbf{Z}_i \rangle
\ \simeq \ q_u^* V_i + \sigma \langle \widebar{\bbf{u}}, \bbf{Z}_i \rangle.
$$
The posterior mean $\widebar{\bbf{u}}$ is not expected to depend much on the particular point $\bbf{Y}_i$ and therefore on $\bbf{Z}_i$. This gives that the random vectors $\widebar{\bbf{u}}$ and  $\bbf{Z}_i$ are approximately independent. Hence the distribution of $\langle \widebar{\bbf{u}}, \bbf{Z}_i \rangle$ is roughly $\cN(0,q_u^*)$ (we recall that $\|\widebar{\bbf{u}} \|^2 \simeq q_u^*$). We get
\begin{equation}\label{eq:approx_law}
\frac{1}{\sigma \sqrt{q_u^*}}
\langle \widebar{\bbf{u}}, \bbf{Y}_i \rangle 
\simeq \sqrt{q_u^*/\sigma^2} V_i + Z
\end{equation}
in law, where $Z \sim \cN(0,1)$. The best estimator $\what{v}_i$ (in terms of MSE) one can then construct using $\langle \widebar{\bbf{u}}, \bbf{Y}_i \rangle$ achieves a MSE of
$$
\E [(V_i - \what{v}_i)^2] \simeq \mmse_v(q_u^*/\sigma^2).
$$
We assumed that $\langle \widebar{\bbf{u}}, \bbf{Y}_i \rangle$ is a sufficient statistic for estimating $V_i$, therefore $\what{v}_i = \widebar{v}_i$.
For all the $\eta N$ indices $i$ such that $S_i = V_i$ we have obviously $\widebar{v}_i = V_i$. Hence
\begin{align*}
\frac{1}{N} \E \|\bbf{V} - \widebar{\bbf{v}} \|^2
&= 
\frac{1}{N} \sum_{i | S_i = 0} \E \big[(V_i - \widebar{v}_i)^2\big]
=
\frac{1}{N} \sum_{i | S_i = 0} \E \big[(V_i - \what{v}_i)^2\big]
\\
&\simeq
\frac{1}{N} \sum_{i | S_i = 0}\mmse_v(q_u^*/\sigma^2)
\simeq
(1-\eta)\mmse_v(q_u^*/\sigma^2).
\end{align*}
Since we have $\frac{1}{N} \E \|\bbf{V} - \widebar{\bbf{v}} \|^2 \simeq 1 - q_v^*$, we get
$$
1 - q_v^* \simeq (1-\eta)\mmse_v(q_u^* /\sigma^2).
$$

We can do the same reasoning with $\widebar{\bbf{u}}$ instead of $\widebar{\bbf{v}}$. We denote by $\bbf{R}_i$ (resp. $\tilde{\bbf{Z}}_i$) the $i$-th row of the matrix $\bbf{Y}$ (resp. $\bbf{Z}$), so that we have
$$
\bbf{R}_i = U_i\bbf{V} + \sigma \tilde{\bbf{Z}}_i.
$$
Hence taking the scalar product with $\frac{1}{N}\widebar{\bbf{v}}$ gives
$$
\frac{1}{\sqrt{N\sigma^2 q^*_v}}\langle \widebar{\bbf{v}},\bbf{R}_i\rangle \simeq \sqrt{\frac{\alpha q^*_v}{\sigma^2}}\sqrt{D} U_i + Z,
$$
in law, where $Z \sim \cN(0,1)$. Recall that $D\E[U_i^2]=1$.
Making the same assumption as above, the best estimator $\what{u}_i$ one can construct using $\langle \widebar{\bbf{v}},\bbf{R}_i\rangle$ achieves a MSE of
$$
\E[D(U_i - \what{u}_i)^2] \simeq\mmse_u(\alpha q_v^* /\sigma^2),
$$
where $\mmse_u(\gamma) = \E [(U-\E[U|\sqrt{\gamma}U + Z])^2]$ for $U,Z \iid \cN(0,1)$.
This leads to
$$
\E[\| \bbf{U}-\widebar{\bbf{u}}\|^2] = 1- q_u^* \simeq \mmse_u(\alpha q_v^* /\sigma^2),
$$

As shown in Section \ref{sec:gauss}, we have
$$
\mmse_u(\gamma) = \frac{1}{1+\gamma}.
$$
We conclude that $(q_u^*,q_v^*)$ satisfies the following fixed point equations:
\begin{align}
	q_v^* &= 1 - (1-\eta)\mmse_v(q_u^* /\sigma^2) \label{eq:SE1} \\
	q_u^* &= \frac{\alpha q_v^*}{\sigma^2 + \alpha q_v^*}. \label{eq:SE2}
\end{align}
We introduce the following mutual information
\begin{equation}\label{eq:def_mi_scalar}
\mathsf{i}_v(\gamma) = I(V_0; \sqrt{\gamma} V_0 + Z_0)
\end{equation}
where $V_0 \sim \Unif(-1,+1)$ and $Z_0 \sim \cN(0,1)$ are independent. An elementary computation leads to (see Section \ref{sec:gauss})
\begin{equation}\label{eq:exp_mi_scalar}
\mathsf{i}_v(\gamma) = \gamma - \E \log \cosh(\sqrt{\gamma} Z_0 + \gamma).
\end{equation}
By the ``I-MMSE'' Theorem from \cite{guo2004mutual}, $\mathsf{i}_v$ is related to $\mmse_v$:
\begin{equation}\label{eq:i_mmse}
\mathsf{i}_v'(\gamma) = \mmse_v(\gamma).
\end{equation}
Let us compute the derivative of $f_{\alpha,\sigma,\eta}$ defined by \eqref{eq:def_f}, using \eqref{eq:i_mmse}:
$$
f'_{\alpha,\sigma,\eta}(q)
= \frac{\alpha}{2 \sigma^2} (1-\eta) \mmse_v(q/\sigma^2) - \frac{\alpha}{2\sigma^2} + \frac{q}{2(1-q)}.
$$
Using \eqref{eq:SE1}-\eqref{eq:SE2}, one verifies easily that $f'_{\alpha,\sigma,\eta}(q_u^*) = 0$.
By Proposition \ref{prop:minimizer} (proved in Section \ref{sec:tech}),
$f_{\alpha,\sigma,\eta}$ admits a unique critical point on $[0,1)$ which is its unique minimizer: $q_u^*$ is therefore the minimizer of $f_{\alpha,\sigma,\eta}$.

If we now want to estimate $V_{\rm new}$ from $\bbf{Y}, \bbf{S}$ and $\bbf{Y}_{\rm new}$ we assume, as above that $\langle \widebar{\bbf{u}}, \bbf{Y}_{\rm new} \rangle$ is a sufficient statistic. As for \eqref{eq:approx_law}, we have
$$
\frac{1}{\sigma \sqrt{q_u^*}}
\langle \widebar{\bbf{u}}, \bbf{Y}_{\rm new} \rangle 
\simeq \sqrt{q_u^*/\sigma^2} V_{\rm new} + Z
$$
in law, where $Z \sim \cN(0,1)$ is independent of $V_{\rm new}$. The Bayes classifier is then
\begin{eqnarray}
\label{eq:hatv}\what{v} = \sign(\langle \overline{\bbf{u}}, \bbf{Y}_{\rm new}\rangle),
\end{eqnarray}
hence
$$
R^*_D(\eta) = \P (V_{\rm new} \neq \what{v}) \simeq 1 -\Phi(\sqrt{q_u^*}/ \sigma),
$$
which is the statement of our main Theorem \ref{th:main} above.

\section{Proof sketch}\label{sec:sketch}

From now we simply write $q^*$ instead of $q^*(\alpha,\sigma,\eta)$.
The next theorem computes the limit of the log-likelihood ratio.
\begin{theorem}\label{th:ratio}
	Conditionally on $V_{\rm new} = \pm 1$,
	\begin{align*}
		\log
	\frac{P(V_{\rm new} = +1 | \bbf{Y},\bbf{S},\bbf{Y}_{\rm new})}{P(V_{\rm new} = -1 | \bbf{Y},\bbf{S},\bbf{Y}_{\rm new})}
	\xrightarrow[N,D \to \infty]{(d)}
	\cN(\pm 2q^*/\sigma^2, 4q^*/\sigma^2).
	\end{align*}
\end{theorem}
Before sketching the proof of this theorem, we show how Theorem \ref{th:main} follows from Theorem \ref{th:ratio} above.
The optimal estimator for $V_{\rm new}$ is given by the sign of the log-likelihood ratio, hence we get:
\begin{eqnarray*}
\lim_{N\to \infty} 1-R^*_D(\eta) = P\left( \cN( 2q^*/\sigma^2, 4q^*/\sigma^2)>0\right) = \Phi(\sqrt{q^*}/\sigma)
\end{eqnarray*}

\begin{proof}
	Let us look at the posterior distribution of $V_{\rm new}, \bbf{U}$ given $\bbf{Y}, \bbf{S}, \bbf{Y}_{\rm new}$, i.e.
	From Bayes rule we get
	\begin{align*}
	\frac{P(V_{\rm new} = +1 | \bbf{Y},\bbf{S},\bbf{Y}_{\rm new})}{P(V_{\rm new} = -1 | \bbf{Y},\bbf{S},\bbf{Y}_{\rm new})}
	&= \frac{\int \exp\big(-\frac{1}{2 \sigma^2}\| \bbf{Y}_{\rm new} - \bbf{u}\|^2\big) dP(\bbf{u}|\bbf{Y},\bbf{S})}{\int \exp\big(-\frac{1}{2 \sigma^2}\| \bbf{Y}_{\rm new} + \bbf{u}\|^2\big) dP(\bbf{u}|\bbf{Y},\bbf{S})}
	\\
	&= \frac{\int \exp\big(\frac{1}{\sigma^2}\langle \bbf{Y}_{\rm new} , \bbf{u}\rangle \big) dP(\bbf{u}|\bbf{Y},\bbf{S})}{\int \exp\big(-\frac{1}{\sigma^2}\langle \bbf{Y}_{\rm new} , \bbf{u}\rangle \big) dP(\bbf{u}|\bbf{Y},\bbf{S})}
	\end{align*}
	Let $\widebar{\bbf{u}} = \E[\bbf{U}|\bbf{Y},\bbf{S}]$. The following lemma is proved in the supplementary material, see Section \ref{sec:tech}.
	\begin{lemma}\label{lem:lim_overlaps}
Let $\bbf{u}^{(1)}, \bbf{u}^{(2)}$ be i.i.d.\ samples from the posterior distribution of $\bbf{U}$ given $\bbf{Y},\bbf{S}$, independently of everything else. Then
		$$
		\|\widebar{\bbf{u}}\|^2 \xrightarrow[N,D \to \infty]{} q^*, \quad
		\langle \widebar{\bbf{u}}, \bbf{u}^{(1)} \rangle \xrightarrow[N,D \to \infty]{} q^*, \quad
		\langle \bbf{U}, \bbf{u}^{(1)} \rangle \xrightarrow[N,D \to \infty]{} q^*, \quad
		\langle \bbf{u}^{(1)}, \bbf{u}^{(2)} \rangle \xrightarrow[N,D \to \infty]{} q^*.
		$$
	\end{lemma}

		For $v \in \{-1,+1\}$ we define
		\begin{align*}
			A_N(v) &= 
		\int \exp\big(\frac{v}{\sigma^2}\langle \bbf{Y}_{\rm new} , \bbf{u}\rangle \big) dP(\bbf{u}|\bbf{Y},\bbf{S})
		\\
			B_N(v) &= \exp \Big( \frac{v}{\sigma} \langle \widebar{\bbf{u}}, \bbf{Z}_{\rm new} \rangle + \frac{v V_{\rm new}}{\sigma^2} q^* \Big).
		\end{align*}
                Using Lemma \ref{lem:lim_overlaps}, we prove the following lemma in Section \ref{sec:tech}
				\begin{lemma}\label{lem:ll}
		For $v = \pm 1$, $A_N(v) - B_N(v) \xrightarrow[N,D \to \infty]{L^2} 0$.
	                        \end{lemma}
		Since $| \log A_N(v) - \log B_N(v) | \leq (A_N(v)^{-1} + B_N(v)^{-1}) (A_N(v) - B_N(v))$, we have by Cauchy-Schwarz inequality:
		$$
		\E | \log A_N(v) - \log B_N(v) | \leq  \sqrt{2} \, \E \big[A_N(v)^{-2} + B_N(v)^{-2}\big]^{1/2}  \E\big[(A_N(v) - B_N(v))^2 \big]^{1/2} \xrightarrow[N,D \to \infty]{} 0,
		$$
		using Lemma \ref{lem:ll} (one can verify easily that the first term of the product above is $O(1)$).
		We get $\log A_N(v) - \log B_N(v) \xrightarrow[N,D \to \infty]{L^1} 0$, hence
	\begin{align*}
		\log
	\frac{P(V_{\rm new} = +1 | \bbf{Y},\bbf{S},\bbf{Y}_{\rm new})}{P(V_{\rm new} = -1 | \bbf{Y},\bbf{S},\bbf{Y}_{\rm new})}
	- \Big(\frac{2}{\sigma} \langle \widebar{\bbf{u}},\bbf{Z}_{\rm new}\rangle + \frac{2}{\sigma^2} q^*V_{\rm new} \Big)
	\xrightarrow[N,D \to \infty]{L^1} 0.
	\end{align*}
	$\widebar{\bbf{u}}$ is independent of $(\bbf{V}_{\rm new}, \bbf{Z}_{\rm new})$ and by Lemma \ref{lem:lim_overlaps} we have $\|\widebar{\bbf{u}}\|^2 \to q^*$.
	Consequently $\langle \widebar{\bbf{u}}, \bbf{Z}_{\rm new} \rangle \xrightarrow[N,D \to \infty]{(d)} \cN(0,q^*)$ and we conclude:
	\begin{align*}
		\log
	\frac{P(V_{\rm new} = +1 | \bbf{Y},\bbf{S},\bbf{Y}_{\rm new})}{P(V_{\rm new} = -1 | \bbf{Y},\bbf{S},\bbf{Y}_{\rm new})}
	\xrightarrow[N,D \to \infty]{(d)}
\frac{2}{\sigma} Z_0 + \frac{2}{\sigma^2} q^*V_{\rm new}
	\end{align*}
	where $Z_0 \sim \cN(0,q^*)$ is independent of $V_{\rm new}$.
\end{proof}

\section{Conclusion}\label{sec:con}

We analyzed a simple high-dimensional Gaussian mixture model in a semi-supervised setting and computed the associated Bayes risk.
In our model, we are able to compute the best possible accuracy of semi-supervised learning using both labeled and unlabeled data as well as the best possible performances of supervised learning using only the labeled data and unsupervised learning using all data but without any label. This allows us to quantify the added value of unlabeled data. When the clusters are well separated (probably the most realistic setting), we find that the value of unlabeled data is dominating. Labeled data can almost be ignored as unsupervised learning achieved roughly the same performance as semi-supervised learning. Nevertheless, using a few labeled data is often very helpful in practice as shown by the recent MixMatch algorithm \cite{berthelot2019mixmatch}.

We believe our main Theorem \ref{th:main} gives new insights for semi-supervised learning and we designed our model with a focus on simplicity. However, our proof technique is very general and can handle a much more complex model. For example, we can deal with classes of different sizes by changing the prior of $V_{\rm new}$. 
Another extension for which our proof carries over consists in modifying the channel for the side information. Here, we considered the erasure channel corresponding to the standard SSL setting but our proof will still work for other channel like the binary symmetric channel or the Z channel corresponding to a setting with noisy labels.

\bibliography{semisupervised-preprint}
\bibliographystyle{abbrvnat}

\section{Supplementary material}\label{sec:supp}
\subsection{Gaussian channel}\label{sec:gauss}
We give here some easy computation for the Gaussian channel:
$$
Y = \sqrt{\gamma}U+Z,
$$
where $Z \sim \cN(0,1)$ is independent of $U$.

We first consider the case where $U\sim \cN(0,1)$. We define $\mmse_u(\gamma) = \E\left[ \left(U- \E\left[U|Y\right]\right)^2\right]$. Since, we are dealing with Gaussian random variables, $\E\left[ U|Y\right]$ is simply the orthogonal projection of $U$ on $Y$:
$$
E\left[ U|Y\right] = \frac{\E\left[ UY\right]}{\E\left[ Y^2\right]}Y = \frac{\sqrt{\gamma}}{1+\gamma}Y.
$$
Hence, we have
$$
\mmse_u(\gamma) = \E\left[ \left(U-\frac{\gamma}{1+\gamma}U-\frac{\sqrt{\gamma}}{1+\gamma}Z\right)^2\right] = \frac{1}{1+\gamma}.
$$
Thanks to the I-MMSE relation \cite{guo2004mutual}, we have $\frac{1}{2}\mmse_u(\gamma) = \frac{\partial}{\partial \gamma}I(U;Y)$. For $\gamma = 0$, $U$ and $Y$ are independent: $I(U;Y)_{\gamma = 0} = 0$, so that we get
$$
I(U;Y) = \frac{1}{2} \log(1+\gamma).
$$

We now consider the case where $U\sim\Unif(-1,+1)$. We define $\mathsf{i}_v(\gamma) = I(U;Y)$. Recall that
$$
I(U;Y) = \E \log \frac{dP_{(U,Y)}}{dP_U\otimes dP_Y}(U,Y).
$$
And here, we have
$$
\frac{dP_{(U,Y)}}{dP_U\otimes dP_Y}(U,Y) = \frac{e^{-1/2(Y-\sqrt{\gamma} U)^2}}{\int e^{-1/2(Y-\sqrt{\gamma} u)^2}dP_U(u)}.
$$
Hence, we have
\begin{eqnarray*}
  \mathsf{i}_v(\gamma) &=& -\E \log \int dP_U(u) \exp\left(\sqrt{\gamma}(u-U)Y \right)\\
  &=& \sqrt{\gamma}\E[UY] - \E \log \cosh(\sqrt{\gamma} Y)\\
    &=& \gamma -\E\log \cosh\left(\sqrt{\gamma} Z+\gamma\right).
\end{eqnarray*}
Thanks to the I-MMSE relation, we have:
\begin{eqnarray*}
  \frac{1}{2}\mmse_v(\gamma) &=& \mathsf{i}'_v(\gamma)
  = 1-\E\left[ \left(\frac{1}{2\sqrt{\gamma}}Z+1\right)\tanh\left(\sqrt{\gamma}Z+\gamma\right)\right]\\
  &=& 1-\E\tanh\left(\sqrt{\lambda}Z+\lambda \right) -\frac{1}{2}\E\tanh'\left(\sqrt{\lambda}Z+\lambda \right)\\
  &=& \frac{1}{2}-\E\tanh\left(\sqrt{\lambda}Z+\lambda \right) +\frac{1}{2}\E\tanh^2\left(\sqrt{\lambda}Z+\lambda \right)\\
  &=& \frac{1}{2}\left(1- \E\tanh\left(\sqrt{\lambda}Z+\lambda \right)\right),
\end{eqnarray*}
so that we have $\mmse_v(\gamma)=1- \E\tanh\left(\sqrt{\lambda}Z+\lambda \right)$.

\subsection{Convergence of the mutual information}
\begin{theorem}\label{th:mi}
	For all $\alpha, \sigma >0$, $\eta \in (0,1]$,
	\begin{equation}\label{eq:RS}
		\frac{1}{N} I\big(\bbf{U},\bbf{V} ; \bbf{Y} \big| \bbf{S} \big)
	\xrightarrow[N,D \to \infty]{}
	\min_{q \in [0,1)} f_{\alpha,\sigma,\eta}(q).
	\end{equation}
	Further, this minimum is achieved at a unique point $q^*(\alpha,\sigma,\eta)$ and
	\begin{equation}\label{eq:lim_overlap}
	\langle \bbf{u},\bbf{U} \rangle \xrightarrow[N,D \to \infty]{} q^*(\alpha,\sigma,\eta),
	\end{equation}
	where $\bbf{u}$ is a sample from the posterior distribution of $\bbf{U}$ given $\bbf{Y},\bbf{S}$, independently of everything else.
\end{theorem}
\begin{proof}
	The limit \eqref{eq:RS} was proved in \cite{miolane2017fundamental} in the case $\eta = 0$. The proof can however be straightforwardly adapted to the case $\eta \neq 0$ and leads to
	$$
		\frac{1}{N} I\big(\bbf{U},\bbf{V} ; \bbf{Y} \big| \bbf{S} \big)
	\xrightarrow[N,D \to \infty]{}
	\inf_{q_u \in [0,1]} 
	\sup_{q_v \in [0,1]}
	\alpha (1-\eta) \mathsf{i}_v(q_u / \sigma^2) + \mathsf{i}_u(\alpha q_v / \sigma^2) + \frac{\alpha}{2 \sigma^2}(1-q_u)(1-q_v),
	$$
	where $\mathsf{i}_u(\gamma) = \frac{1}{2} \log(1+\gamma)$. The supremum in $q_v$ can be easily computed, leading to:
\begin{align*}
	\sup_{q_v \in [0,1]}
	\alpha (1-\eta)\mathsf{i}_v(q_u / \sigma^2) &+ \mathsf{i}_u(\alpha q_v / \sigma^2) + \frac{\alpha}{2 \sigma^2}(1-q_u)(1-q_v)
	\\
												&=
	\alpha (1-\eta)\mathsf{i}_v(q_u / \sigma^2) + \frac{\alpha}{2\sigma^2}(1-q_u) - \frac{1}{2}\big(q_u  + \log(1-q_u)\big).
\end{align*}
This proves \eqref{eq:RS}. 
The fact that $f_{\alpha,\sigma,\eta}$ admits a unique minimizer $q_u^*(\alpha,\sigma,\eta)$ comes from Proposition \ref{prop:minimizer}.


From the limit of the mutual information, one gets the limits of minimal mean squared errors (MMSE) using the ``I-MMSE'' relation \cite{guo2004mutual}:
	\begin{align*}
	&\E \big\| \bbf{U} \bbf{U}^{\sT} - \E[\bbf{U} \bbf{U}^{\sT} | \bbf{Y}, \bbf{S}] \big\|^2 \xrightarrow[N,D \to \infty]{} 1 - q_u^*(\alpha,\sigma,\eta)^2.
	\end{align*}
	Let $\bbf{u}$ be a sample from the posterior distribution of $\bbf{U}$ given $\bbf{Y},\bbf{S}$, independently of everything else. Then we deduce
\begin{align*}
	&\E \big[\langle \bbf{U},\bbf{u} \rangle^2  \big]
	\xrightarrow[N,D \to \infty]{} q_u^*(\alpha,\sigma,\eta)^2
\end{align*}
In order to show that $\langle \bbf{u} , \bbf{U} \rangle \xrightarrow[n \to \infty]{} q_u^*(\alpha,\sigma,\eta)$ it remains to show that
\begin{equation}\label{eq:limsup4}
\limsup_{N,D \to \infty} \E \big[\langle \bbf{U},\bbf{u} \rangle^2  \big] \leq q_u^*(\alpha,\sigma,\eta)^4.
\end{equation}
This can be done (as in \cite{barbier2019optimal}) by adding a small amount of additional side-information to the model of the form $\bbf{Y} = \sqrt{\epsilon D} \bbf{U}^{\otimes 4} + \bbf{W}$,
where the entries of the tensor $W$ are i.i.d.\ standard Gaussian: $( W_{i_1,i_2,i_3,i_4} )_{1 \leq i_1, i_2, i_3, i_4 \leq D} \iid \cN(0,1)$. We then apply the I-MMSE relation with respect to $\epsilon$ to obtain \eqref{eq:limsup4}.
\end{proof}

\subsection{Technical lemmas}\label{sec:tech}
We now give the proof of Lemma \ref{lem:lim_overlaps}
	\begin{proof}
		Notice that, by Bayes rule, we have $(\bbf{U},\bbf{u}^{(1)}) \eqlaw (\bbf{u}^{(2)},\bbf{u}^{(1)})$. So we have by \eqref{eq:lim_overlap} $\langle \bbf{U}, \bbf{u}^{(1)} \rangle , \langle \bbf{u}^{(1)}, \bbf{u}^{(2)} \rangle \xrightarrow[N,D \to \infty]{} q^*$.
		Now, by Jensen's inequality:
		$$
		\E \Big[\E\big[\langle \bbf{u}^{(1)}, \bbf{u}^{(2)}\rangle - q^*\big| \bbf{Y},\bbf{S}\big]^2\Big]
		\leq
		\E \Big[\E\big[\langle \bbf{u}^{(1)}, \bbf{u}^{(2)}\rangle - q^*\big| \bbf{Y},\bbf{S},\bbf{u}^{(1)}\big]^2\Big]
		\leq
		\E \big[(\langle \bbf{u}^{(1)}, \bbf{u}^{(2)}\rangle - q^*)^2\big].
		$$
		Since $\E[\langle \bbf{u}^{(1)}, \bbf{u}^{(2)}\rangle| \bbf{Y},\bbf{S}] = \langle \E[\bbf{u}^{(1)}| \bbf{Y},\bbf{S}] , \E[ \bbf{u}^{(2)}| \bbf{Y},\bbf{S}]  \rangle = \|\widebar{\bbf{u}}\|^2$ and $\E[\langle \bbf{u}^{(1)}, \bbf{u}^{(2)}\rangle| \bbf{Y},\bbf{S},\bbf{u}^{(1)}] = \langle \bbf{u}^{(1)}, \E[ \bbf{u}^{(2)}| \bbf{Y},\bbf{S}]  \rangle = \langle \bbf{u}^{(1)}, \widebar{\bbf{u}} \rangle$,
		this leads to
		$$
		\E \big[(\|\widebar{\bbf{u}}\|^2 - q^*)^2\big]
		\leq
		\E \big[(\langle \widebar{\bbf{u}}, \bbf{u}^{(1)}\rangle - q^*)^2\big]
		\leq
		\E \big[(\langle \bbf{u}^{(1)}, \bbf{u}^{(2)}\rangle - q^*)^2\big].
		$$
	\end{proof}

	We now give a proof of Lemma \ref{lem:ll}.	
	\begin{proof}
		In order to  prove that $A_N(v) - B_N(v) \xrightarrow[N,D \to \infty]{L^2} 0$, it suffices to show that $\lim \E [A_N(v)^2] = \lim \E [B_N(v)^2] = \lim \E [A_N(v) B_N(v)]$.
		Let $\bbf{u}^{(1)}, \dots \bbf{u}^{(2)}$ be i.i.d.\ samples from the posterior distribution of $\bbf{U}$ given $\bbf{Y},\bbf{S}$, independently of everything else. Using Lemma \ref{lem:lim_overlaps}, we compute:
		\begin{align*}
			\E [A_N(v)^2] 
	&= \E \exp \big(\frac{v}{\sigma^2} \langle \bbf{Y}_{\rm new}, \bbf{u}^{(1)} + \bbf{u}^{(2)} \rangle \big)
	\\
	&= \E \exp \big(\frac{v}{\sigma} \langle \bbf{Z}_{\rm new}, \bbf{u}^{(1)} + \bbf{u}^{(2)} \rangle + \frac{v}{\sigma^2} V_{\rm new} \langle \bbf{U}, \bbf{u}^{(1)} + \bbf{u}^{(2)} \rangle \big).
		\end{align*}
		Integrating with respect to $\bbf{Z}_{\rm new} \sim \cN(0,\Id_D)$ only, we get
		\begin{align*}
			\E [A_N(v)^2] 
	&= \E \exp \big(\frac{1}{2\sigma^2} \| \bbf{u}^{(1)} + \bbf{u}^{(2)} \|^2  + \frac{v}{\sigma^2} V_{\rm new}\langle \bbf{U}, \bbf{u}^{(1)} + \bbf{u}^{(2)} \rangle \big)
	\\
	&= \E \exp \big(\frac{1}{\sigma^2} \langle \bbf{u}^{(1)} , \bbf{u}^{(2)} \rangle + \frac{1}{\sigma^2} + \frac{v}{\sigma^2} V_{\rm new} \langle \bbf{U}, \bbf{u}^{(1)} + \bbf{u}^{(2)} \rangle \big)
	\\
	&\xrightarrow[N,D \to \infty]{} \E \exp\big( (1+q^*)/\sigma^2 + 2vV_{\rm new} q^* /\sigma^2 \big),
		\end{align*}
		where the last limit follows from Lemma \ref{lem:lim_overlaps}.
		Following the same steps we compute:
		\begin{align*}
			\E [B_N(v)^2] 
	&= \E \exp \Big( \frac{2v}{\sigma} \langle \widebar{\bbf{u}}, \bbf{Z}_{\rm new} \rangle + \frac{2v}{\sigma^2} V_{\rm new} q^* + \frac{1}{\sigma^2}(1-q^*) \Big)
	\\
	&= \E \exp \Big( \frac{2}{\sigma^2} \| \widebar{\bbf{u}}\|^2 + \frac{2v}{\sigma^2} V_{\rm new} q^* + \frac{1}{\sigma^2}(1-q^*)\Big)
	\\
	&\xrightarrow[N,D \to \infty]{} \E \exp\big( (1+q^*)/\sigma^2 + 2v V_{\rm new} q^* /\sigma^2 \big),
		\end{align*}
		and
		\begin{align*}
			\E [A_N(v) B_N(v)] 
	&= \E \exp \big(\frac{v}{\sigma} \langle \bbf{Z}_{\rm new}, \bbf{u}^{(1)} + \widebar{\bbf{u}} \rangle + \frac{v}{\sigma^2}V_{\rm new} \langle \bbf{U}, \bbf{u}^{(1)}  \rangle  + \frac{v}{\sigma^2} V_{\rm new} q^* + \frac{1}{2 \sigma^2}(1-q^*)\big)
	\\
	&= \E \exp \big(\frac{1}{2\sigma^2} \| \bbf{u}^{(1)} + \widebar{\bbf{u}} \|^2  + \frac{v}{\sigma^2}V_{\rm new} \langle \bbf{U}, \bbf{u}^{(1)} \rangle + \frac{v}{\sigma^2} V_{\rm new} q^* + \frac{1}{2 \sigma^2}(1-q^*)\big)
	\\
	&= \E \exp \big(\frac{1}{\sigma^2} \langle \bbf{u}^{(1)} , \widebar{\bbf{u}} \rangle  + \frac{1}{2 \sigma^2} \|\widebar{\bbf{u}} \|^2 + \frac{1}{2\sigma^2} + \frac{v}{\sigma^2} V_{\rm new}\langle \bbf{U}, \bbf{u}^{(1)} \rangle + \frac{v}{\sigma^2} V_{\rm new} q^* + \frac{1}{2 \sigma^2}(1-q^*)\big)
	\\
	&\xrightarrow[N,D \to \infty]{} \E \exp\big( (1+q^*)/\sigma^2 + 2vV_{\rm new} q^* /\sigma^2 \big).
		\end{align*}
		The three limits above are the same, the Lemma is proved.
	\end{proof}

\begin{proposition}\label{prop:minimizer}
	For all $\alpha, \sigma >0$ and all $\eta \in (0,1]$, the function $f_{\alpha,\sigma,\eta}$ admits a unique critical point which is its unique minimizer on $[0,1)$.
\end{proposition}
\begin{proof}
	Recall that $\mathsf{i}_v(\gamma) = \gamma - \E \log \cosh(\sqrt{\gamma} Z + \gamma)$, where $Z \sim \cN(0,1)$. A computation gives $\mathsf{i}'_v(\gamma) = \frac{1}{2} (1 - \E \tanh(\sqrt{\gamma}Z + \gamma))$. 
	We define,
	$$
	h(\gamma) = \E \big[\tanh(\sqrt{\gamma} Z + \gamma) \big],
	$$
	where $Z \sim \cN(0,1)$.
	Hence
$$
f'_{\alpha,\sigma,\eta}(q)
= \frac{\alpha}{2 \sigma^2} (1-\eta) (1- h(q/\sigma^2)) - \frac{\alpha}{2\sigma^2} + \frac{q}{2(1-q)}.
$$
The critical points of $f_{\alpha,\sigma,\eta}$ are solution of
$$
q = F(q) \defeq \frac{\alpha (\eta + (1-\eta) h(q/\sigma^2))}{\sigma^2 + \alpha (\eta + (1-\eta) h(q/\sigma^2))}.
$$
As proved in \cite[Lemma 6.1]{deshpande2016asymptotic}, the function $h$ is concave. This gives that $F$ is concave. Since $F$ is upper-bounded by $1$ and $F(0) = \frac{\alpha \eta}{\sigma^2 + \alpha \eta} > 0$, we get that $F$ admits a unique fixed point on $[0,1]$.
The function $f_{\alpha,\sigma,\eta}$ admits therefore a unique critical point on $[0,1)$ which is necessarily a minimum since $f'_{\alpha,\sigma,\eta}(0) = - \frac{\eta \alpha}{2 \sigma}$ and $\lim_{q \to 1} f'_{\alpha,\sigma,\eta}(q) = + \infty$
\end{proof}

\end{document}